\documentclass[runningheads,a4paper]{llncs}

\newcommand{\removed}[1]{}
\usepackage{graphicx}
\usepackage{amssymb,amsmath}

\usepackage{amsthm}
\usepackage{verbatim}

\newtheorem{clm}{Claim}

\newcommand{\EE}{\mathbb{E}}
\newcommand{\RR}{\mathbb{R}}
\newcommand{\NN}{\mathbb{N}}
\newcommand{\PP}{\mathbb{P}}
\newcommand{\FF}{\mathcal{F}}
\newcommand{\HH}{\mathcal{H}}
\newcommand{\GG}{\mathcal{G}}
\newcommand{\sign}{\text{sign}}
\newcommand{\inner}[1]{\langle #1\rangle}
\newcommand{\norm}[1]{\left\lVert#1\right\rVert}
\newcommand{\abs}[1]{\left\lvert#1\right\rvert}
\DeclareMathOperator{\conv}{conv}
\newcommand{\sm}{\text{sim}}

\usepackage{url}
\urldef{\mails}\path|{bneyshabur, yury, nati}@ttic.edu|    
\newcommand{\keywords}[1]{\par\addvspace\baselineskip
\noindent\keywordname\enspace\ignorespaces#1}

\begin{document}

\mainmatter  

\title{Clustering, Hamming Embedding,\\
  Generalized LSH and the Max Norm}


%
%
\author{Behnam Neyshabur \and Yury Makarychev \and Nathan Srebro}
\authorrunning{Neyshabur et al.}

\institute{Toyota Technological Institute at Chicago,\\
\mails\\}

%
%

\toctitle{Clustering, Hamming Embedding, Generalized LSH and the Max Norm}
\maketitle
\begin{abstract}
 We study the convex relaxation of clustering and hamming embedding,
  focusing on the asymmetric case (co-clustering and asymmetric
  hamming embedding), understanding their relationship to LSH as
  studied by \cite{charikar02} and to the max-norm ball, and the
  differences between their symmetric and asymmetric versions.
  \keywords{Clustering, Hamming Embedding, LSH, Max Norm}
\end{abstract}

\section{Introduction}
Convex relaxations play an important role in designing efficient
learning and recovery algorithms, as well as in statistical learning
and online optimization.  It is thus desirable to understand the
convex hull of hypothesis sets, to obtain tractable relaxation to
these convex hulls, and to understand the tightness of such
relaxations.

In this paper we consider convex relaxations of two important
problems, namely clustering and hamming embedding, and study the
convex hulls of the corresponding hypothesis classes: of cluster
incidence matrices and of similarity measures with a short hamming
embedding.  In section \ref{s:clustering} we introduce these classes formally, and
understand the relationship between them, showing how hamming embedding
can be seen as a generalization of clustering. In section \ref{s:lsh} we
discuss their convex hull and its relationship to notion of Locality
Sensitive Hashing (LSH) as studied by \cite{charikar02}. There has been several studies on different aspects of LSH (e.g.\cite{indyk98,datar04,chierichetti10}).

More specifically, we focus on the asymmetric versions of these
classes, which correspond to co-clustering (e.g. \cite{dhillon03,banerjee04}) and
asymmetric hamming embedding as recently introduced by \cite{neyshabur13}.  We
define the corresponding notion of an Asymmetric LSH, and show how it
could be much more powerful then standard (symmetric) LSH (section
\ref{s:asym}).  

Our main conclusion is that the convex hull of asymmetric clustering
and hamming embedding is tightly captured by a shift-invariant
modification of the max-norm---a tractable SDP-representable
relaxation (Theorem \ref{thm:ineq} in section \ref{s:relaxation}).  We contrast this with
the symmetric case, in which the corresponding SDP relaxation is not
tight, highlighting an important distinction between symmetric and
asymmetric clustering, embedding and LSH. 

\section{Clustering and Hamming Embedding}
\label{s:clustering}
In this section we introduce the problems of clustering and hamming
embedding, providing a unified view of both problems, with hamming
embedding being viewed as a direct generalization of clustering.  Our
starting point, in any case, is a given similarity function
$\sm:S\times S\rightarrow [-1,+1]$ over a (possibly infinite) set of
objects $S$.  ``Clustering'', as we think of it here, is the problem
of partitioning the elements of $S$ into disjoint clusters so that 
items in the same cluster are similar while items in different
clusters are not similar.  ``Hamming Embedding'' is the problem of
embedding $S$ into some hamming space such that the similarity between
objects is captured by the hamming distance between their mappings.

\subsection{Clustering}

We represent a clustering of $S$ as a mapping $h:S\rightarrow \Gamma$,
where $\Gamma$ is a discrete alphabet.  We can think of $h$ as a
function that assigns a cluster identity to each element, where the
meaning of the different identities is arbitrary.  The alphabet
$\Gamma$ might have a fixed finite cardinality $\left|\Gamma\right|=k$, if we
would like to have a clustering with a specific number of clusters.  E.g., a
binary alphabet corresponds to standard graph partitioning into two
clusters.  If $|\Gamma| = k$, we can assume that $\Gamma=[k]$. The alphabet
$\Gamma$ might be infinitely countable (e.g. $\Gamma=\mathbb{N}$),
in which case we are not constraining the number of clusters.

The {\em cluster incidence function} $\kappa_h:S\times S \rightarrow
\{\pm 1\}$ associated with a clustering $h$ is defined as
$\kappa_h(x,y)=1$ if $h(x)=h(y)$ and $\kappa_h(x,y)=-1$ otherwise.
For a finite space $S$ of cardinality $n=\left|S\right|$ we can think
of $\kappa_h \in \{\pm 1\}^{n\times n}$ as a permuted block-diagonal
matrix.  We denote the set of all valid cluster incidence functions
over $S$ with an alphabet of size $k$ (i.e.~with at most $k$ clusters)
as $ M_{S,k} = \left\{ \kappa_h \;\middle|\; h:S\rightarrow [k]
  \right\}$, where $k=\infty$ is allowed.

With this notion in hand, we can think of clustering as a problem of
finding a cluster incidence function $\kappa_h$ that approximates a
given similarity $sim$, as quantified by objectives 
$\min \EE_{x,y}[\abs{\kappa_h(x,y)-\sm(x,y)}]$ or $\max \EE_{x,y}[\sm(x,y)
\kappa_h(x,y)]$ (this is
essentially the correlation clustering objective).  Since objectives
themselves are convex in $\kappa$, but the constraint that $\kappa$ is
a valid cluster incidence function is not a convex constraint, a
possible approach is to relax the constraint that $\kappa$ is a valid
cluster incidence function, or in the finite case, a cluster incidence
matrix.  This is the approach taken by, e.g.~\cite{jalali11,jalali12},
who relax the constraint to a trace-norm and max-norm constraint
respectively.  One of the questions we will be exploring here is
whether this is the tightest relaxation possible, or whether there is
a significantly tighter relaxation.

\subsection{Hamming Embedding and Binary Matrix Factorization}

In the problem of binary hamming embedding (also known as binary
hashing), we want to find a mapping from each object $x\in S$ to
binary string $b(x)\in \{\pm 1\}^d$ such that similarity between strings is
approximated by the hamming distance between their images:
\begin{equation}\label{eq:simbyhamming}
\sm(x,y)\approx 1-\frac{2\delta_{\text{Ham}}(b(x),b(y))}{d}
\end{equation}
Calculating the hamming distance of two binary hashes is an extremely
fast operation, and so such a hash is useful for very fast computation of
similarities between massive collections of objects.
Furthermore, hash tables can be used to further speed up retrieval of
similar objects. 

Binary hamming embedding can be seen as a generalization of clustering
as follows:  For each position $i=1,\ldots,d$ in the hash, we can
think of $b_i(x)$ as a clustering into two clusters (i.e.~with
$\Gamma=\{\pm 1\}$).  The hamming distance is then an average of the
$d$ cluster incidence functions:
$$ 1-\frac{2\delta_{\text{Ham}}(b(x),b(y))}{d} = \frac{1}{d}
\sum_{i=1}^d \kappa_{b_i}(x,y).$$
Our goal then is to approximate a similarity function by an average of
$d$ binary clusterings.  For $d=1$ this is exactly a binary
clustering.  For $d>1$, we are averaging multiple binary clusterings.

Since we have $\inner{b(x),b(y)}=d-2\delta_{\text{Ham}}(b(x),b(y))$,
we can formulate the binary hashing problem as a binary matrix
factorization where the goal is to approximate the similarity matrix
by a matrix of the form $RR^{\top}$, where $R$ is a $d$-dimensional
binary matrix:

\begin{equation}
\label{eq:realsim}
\begin{aligned}
\underset{R}{\min}&
& &\sum_{ij}\text{err}(\sm(i,j),X(i,j))\\
\text{s.t}&
& & X=RR^{\top}\\
& &  &R\in \{\pm 1\}^{n\times d}\\
\end{aligned}
\end{equation}
where $err(x,y)$ is some error function such as $err(x,y)=|x-y|$.

Going beyond binary clustering and binary embedding, we can consider
hamming embeddings over larger alphabets.  That is, we can consider mappings
$b:S\rightarrow \Gamma^d$, where we aim to approximate the similarity
as in \eqref{eq:simbyhamming}, recalling that the hamming distance
always counts the number of positions in which the strings disagree.
Again, we have that the length $d$ hamming embeddings over a (finite
or infinitely countable) alphabet $\Gamma$ correspond to averages of
$d$ cluster incidence matrices over the same alphabet $\Gamma$.

\removed{
\subsection{Binary Similarity function: Threshold approximation}

Of particular interest are binary similarity functions 
$\sm(x,y)\rightarrow\{\pm 1\}$. In this case, we seek a hamming
embedding 
that enables us to capture the similarity value by
proper thresholding of the hamming distance. The binary matrix
factorization formulation of this problem can be either defined as the problem of 
of finding a binary code of a given length that minimizes the number of wrong predictions:
\begin{equation}
\label{eq:binsim}
\begin{aligned}
& \underset{R,\theta}{\min} 
& &-\sum_{ij} \sign[ \sm(i,j) ( X(i,j) - \theta )]\\
& \text{s.t}
& & X=RR^{\top}\\
& &  &R\in \{\pm 1\}^{n\times d}\\
\end{aligned}
\end{equation}
or as the problem of finding the minimum code length that can capture the similarity:
\begin{equation}
\label{eq:minl}
\begin{aligned}
& \underset{R,\theta}{\min} 
& &d\\
& \text{s.t}
& & X=RR^{\top}\\
& &  &R\in \{\pm 1\}^{n\times d}\\
& & &\sm(i,j) ( X(i,j) - \theta ) \geq 1 & \forall i,j\\
\end{aligned}
\end{equation}
}

\section{Locality Sensitive Hashing Schemes}
\label{s:lsh}
Moving on from a finite average of clusterings, with a fixed number of
components, as in hamming embedding, to an infinite average, we arrive
at the notion of LSH as studied by \cite{charikar02}.

Given a collection $S$ of objects, an alphabet $\Gamma$ and a
similarity function $\sm:S\times S \rightarrow [-1,1]$ such that for
any $x\in S$ we have $\sm(x,x)=1$,a {\em locality sensitive
  hashing scheme} ({\bf LSH}) is a probability distribution on the family of
clustering functions (hash functions) $\HH=\{ h:S\rightarrow\Gamma\}$
such that \cite{charikar02}:
\begin{equation}\label{eq:lshdef}
\EE_{h\in\HH}[\kappa_h(x,y)]=\sm(x,y).
\end{equation}
\cite{charikar02} discuss similarity functions $\sm:S\times
S\rightarrow[0,1]$ as so require 
$$\PP_{h\in \HH}[h(x)=h(y)]=\sm(x,y).$$
The definition \eqref{eq:lshdef} is equivalent, except it applies to
the transformed similarity function $2\sm(x,y)-1$.

The set of all locality sensitive hashing schemes with an alphabet of
size $k$ is nothing but the convex hull of the set $M_{S,k}$ of
cluster incidence matrices.

The importance of an LSH, as an object in its own right as studied by
\cite{charikar02}, is that a hamming embedding can be obtained
from an LSH by randomly generating a finite number of hash functions
from the distribution over the family $\HH$. In particular, if we draw
$h_1,\ldots,h_d$ i.i.d.~from an LSH, then the length-$d$ hamming
embedding $b(x)=[h_1(x),\ldots,h_d(x)]$ has expected square error
\begin{equation}\label{eq:embLSH}
E[(\sm(x,y)-\frac{1}{d}\sum \kappa_{h_d}(x,y))^2]\leq\frac{1}{d},
\end{equation}
where the expectation is w.r.t.~the sampling, and this holds for all
$x,y$, and so also for any average over them.

\subsection{$\alpha$-LSH}

If the goal is to obtain an low-error embedding, the requirement
\eqref{eq:lshdef} might be too harsh.  If we are willing to tolerate a
fixed offset between our embedding and the target similarity, we can
instead require that
\begin{equation}\label{eq:alphalshdef}
\alpha \EE_{h\in \HH}[\kappa_h(x,y)] - \theta =\sm(x,y).
\end{equation}
where $\alpha,\theta \in \RR$, $\alpha>0$.  A distribution over $h$
that obeys \eqref{eq:alphalshdef} is called an {\bf $\alpha$-LSH}.  We
can now verify that, for $h_1,\ldots,h_d$ drawn i.i.d.~from an
$\alpha$-LSH, and any $x,y\in S$:
\begin{equation}\label{eq:embLSH}
E\left[\left(\sm(x,y)-(\frac{\alpha}{d}\sum \kappa_{h_d}(x,y)-\theta)\right)^2\right]\leq\frac{\alpha^2}{d}.
\end{equation}
The length of the LSH required to acheive accurate approximation of a
similarity function thus scales quadartically with $\alpha$, and it is
therefor desireable to obtain an $\alpha$-LSH with as low an $\alpha$
as possible (note that $\sm(x,x)=1$, implies $\theta=\alpha-1$, and
so we must allow a shift if we want to allow $\alpha\neq 1$).

Unfortunately, even the requirement \eqref{eq:alphalshdef} of an
$\alpha$-LSH is quite limiting and difficult to obey, as captured by
the following theorem, which is based on lemmas 2 and 3 of
\cite{charikar02}:
\begin{clm}\label{claim:alphaLSH}
  For any finite or countable alphabet $\Gamma$, $k=\abs{\Gamma}\geq
  2$, a similarity function $\sm$ has an $\alpha$-LSH over $\Gamma$ for
  some $\alpha$ if and only if $D(x,y)=\frac{1-\sm(x,y)}{2}$ is embeddable to
  hamming space with no distortion.
\end{clm}
\begin{proof}
Given metric spaces $(X,d)$ and $(X,d')$ any map $f:X\rightarrow X'$ is called a metric embedding. The distortion of such an embedding is defined as:
$$
\beta= \max_{x,y \in X} \frac{d(x,y)}{d'(f(x),f(y))}. \max_{x,y \in X} \frac{d'(f(x),f(y))}{d(x,y)}
$$
We first show that if there exist an $\alpha$-LSH for function $\sm(x,y)$ then $\frac{1-\sm(x,y)}{2}$ is embeddable to hamming space with no distortion. An $\alpha$-LSH for function $\sm(x,y)$ corresponds to an $LSH$ for function $1-\frac{1-\sm(x,y)}{\alpha}$. Using lemma 3 in \cite{charikar02}, we can say that $\frac{1-\sm(x,y)}{\alpha}$ can be isometrically embedded in the Hamming cube which means $1-\sm(x,y)$ can be embedded in Hamming cube with no distortion.
\begin{eqnarray*}
\EE_{h\sim \mathcal{D}_\HH}[\kappa_h(x,y)] &=& 2\PP_{h\sim \mathcal{D}_\HH}[h(x)=h(y)]  - 1\\
&=& 1-2\PP_{h\sim \mathcal{D}_\HH}[h(x)\neq h(y)]\\
&=& 1-2d_H(x,y)\\
&=& 1 - \frac{2d(x,y)}{\beta}\\
&=& \frac{2}{\beta}(\frac{\beta}{2} - d(x,y))\\
&=& \frac{2}{\beta}(1 - d(x,y)+\frac{\beta}{2}-1)\\
&=& \frac{\sm(x,y)+\theta}{\theta+1}\\
\end{eqnarray*}

\end{proof}

\removed{
\begin{lem}
\label{lem:binary}
For any function $\sm(x,y)$ and any $\alpha > 0$, if there is an $\alpha$-LSH with $k$ hashing values for $\sm(x,y)$ then there is a binary $2\alpha$-LSH for $\sm(x,y)$.
\end{lem}
This lemma a simple extension of lemma 2 in \cite{charikar02}. The next lemma states the equivalence of $\alpha$-LSH and embedding into hamming space. A similar argument is made in the proof of lemma 3 in \cite{charikar02}.

\begin{claim}
\label{claim:embedding}
For any finite collection $S$ of objects and any bounded distance function $D:S\times S \rightarrow \RR^+$, the space $(S,D)$ is embeddable to hamming space with no distortion if and only if there exists an $\alpha>0$ such that there is a binary $\alpha$-LSH for the function $\sm(x,y)= 1 - D(x,y)$.
\end{claim}
}

As a result of Claim \ref{claim:alphaLSH}, it  can be shown that given \emph{any} large enough set of low dimensional unit vectors, there is no $\alpha$-LSH for the Euclidian inner product.
\begin{clm}\label{claim:nolsh}
Let $\{x^{(1)},\dots,x^{(n)}\}$ be an arbitrary set of unit vectors 
in the unit sphere. Let $Z_{ij}=\inner{x^{(i)},x^{(j)}}$ for $1\leq i,j \leq n$. If $d < \log_2 n$, then there is no $\alpha$-LSH for $Z$.
\end{clm}
\begin{proof}
According to~\cite{DG62} (see also \cite{buchok10}), 
if $d < \log_2 n$ then in any set of $n$ points in $d$-dimensional Euclidian space, there exist at least three points that form an obtuse triangle. Equivalently,  there exist three vectors $x$, $y$ and $z$ in any set of $n$ different $d$-dimensional unit vectors such that:
$$
\inner{z-x,z-y} < 0
$$
We rewrite the above inequality as:
$$
(1-\inner{z,x}) + (1-\inner{z,y}) < (1-\inner{x,y})
$$
The above inequality implies that the distance measure $\Delta_{ij}=(1-Z_{ij})/2$ is not a metric. Consequently, according to Claim \ref{claim:alphaLSH} since $\Delta_{ij}=(1-Z_{ij})/2$ is not a metric, there is no $\alpha$-LSH for the matrix $Z$.
\end{proof}

As noted by \cite{charikar02} (and stated in claim \ref{claim:nolsh}), we can therefore
unfortunately conclude that there is no $\alpha$-LSH for several
important similarity measures such as the Euclidian inner product,
Overlap coefficient and Dice's coefficient. Note that based on Claim \ref{claim:alphaLSH}, even a finite positive semidefinite similarity matrix is not necessarily $\alpha$-LSHable.

\subsection{Generalized $\alpha$-LSH}

In the following section, we will see how to break the barrier imposed
by Claim \ref{claim:alphaLSH} by allowing asymmetry, highlighting the
extra power asymmetry affords us.  But before doing so, let us
consider a different attempt at relaxing the definition of an
$\alpha$-LSH, motivated by to the work of \cite{charikar04} and
\cite{alon06}: in order to uncouple the shift $\theta$ from the
scaling $\alpha$, we will allow for a different, arbitrary, shift on
the self-similarities $\sm(x,x)$ (i.e.~on the diagonal of $\sm$).

We say that a probability distribution over $\HH=\{
h:S\rightarrow\Gamma\}$ is a {\bf Generalized $\alpha$-LSH}, for
$\alpha>0$ if there exist $\theta,\gamma\in\RR$ such that for all $x,y$:
$$
\alpha \EE_{h\in \HH}[\kappa_{h}(x,y))]=\sm(x,y) + \theta + \gamma 1_{x=y}
$$
With this definition, then any symmetric similarity function, at least over a
finite domain, admits a Generalized $\alpha$-LSH, with a sufficiently
large $\alpha$: 
\begin{clm}\label{claim:ALSHexist}
  For a finite set $S$, $\abs{S}=n$, for any symmetric $\sm:S\times S
  \rightarrow [-1,1]$ with $\sm(x,x)=1$, there exists a Generalized
  $\alpha$-LSH over a binary alphabet $\Gamma$ ($\abs{\Gamma}=2$)
  where $\alpha=O((1 - \lambda_{\min})\log n)$-LSH, and $\lambda_{\min}$ is
  the smallest eigenvalue of the matrix $\sm$.
\end{clm}
\begin{proof}
We observe that $\sm - \lambda_{\min} I$ is a positive semidefinite matrix.
According to~\cite{charikar04}, if a matrix $Z$ with unit diagonal is positive semidefinite, then
there is a probability distribution over a family $\HH$ of hash functions such that for any $x\neq y$:
$$
E_{h \in \HH}[h(x)h(y)] = \frac{Z(x,y)}{C\log n}.
$$
We let $Z(x,y) = (\sm(x,y) - \lambda_{\min} 1_{x=y})/ (1- \lambda_{\min})$. Matrix $Z$ is positive semi-definite and has unit diagonal.
Hence, there is  a probability distribution over a family $\HH$ of hash functions such that
$$
E_{h \in \HH}[h(x)h(y)] = \frac{\sm(x,y) - \lambda_{\min} 1_{x=y}}{C (1 - \lambda_{\min})\log n},
$$
equivalently
$$
(C (1 - \lambda_{\min}) \log n) \cdot \EE_{h\in \HH}[\kappa_{h}(x,y))]=\sm(x,y) - \lambda_{\min} 1_{x=y}.
$$
\end{proof}

It is important to note that $\lambda_{\min}$ could be negative, and as low
as $\lambda_{\min}=-\Omega(n)$.  The required $\alpha$ might therefor be
as large as $\Omega(n)$, yielding a terrible LSH.

\section{Asymmetry}
\label{s:asym}

In order to allow for greater power, we now turn to {\em Asymmetric}
variants of clustering, hamming embedding, and LSH.  

Given two collections of objects $S$, $T$, which might or might not be
identical, and an alphabet $\Gamma$, an {\em asymmetric clustering}
(or co-clustering \cite{dhillon03}) is specified by pair of mappings
$f:S\rightarrow\Gamma$ and $g:T\rightarrow\Gamma$ and is captured by
the asymmetric cluster incidence matrix $\kappa_{f,g}(x,y)$ where
$\kappa_{f,g}(x,y)=1$ if $f(x)=g(y)$ and $\kappa_{f,g}(x,y)=-1$
otherwise.  We denote the set of all valid asymmetric cluster
incidence functions over $S,T$ with an alphabet of size $k$ as $
M_{(S,T),k} = \left\{ \kappa_{f,g} \;\middle|\; f:S\rightarrow [k],
  g:T\rightarrow [k]  \right\}$, where we again also allow $k=\infty$
to correspond to a countable alphabet $\Gamma=\NN$.

Likewise, an asymmetric binary embedding of $S,T$ with alphabet
$\Gamma$ consists of a pair of functions
$f:S\rightarrow\Gamma^d,g:T\rightarrow\Gamma^d$, where we approximate
a similarity as:
\begin{equation}
  \label{eq:asymEmb}
  \sm(x,y)\approx
  1-\frac{2\delta_{\text{Ham}}(f(x),g(y))}{d}=\frac{1}{d}\sum_{i=1}^d \kappa_{f_i,g_i}(x,y).
\end{equation}
That is, in asymmetric hamming embedding, we approximate a similarity
as an average of $d$ asymmetric cluster incidence matrices from $M_{(S,T),k}$.

In a recent work,  \cite{neyshabur13} showed that even when $S=T$ and the similarity
function $\sm$ is a well-behaved symmetric similarity function,
asymmetric binary embedding could be much more powerful in approximating
the similarity, using shorter lengths $d$, both theoretically and
empirically on data sets of interest.  That is, these
concepts are relevant and useful not only in an a-priory asymmetric
case where $S\neq T$ or $\sm$ is not symmetric, but also when the
target similarity {\em is} symmetric, but we allow an asymmetric
embedding.  We will soon see such gaps also when considering the
convex hulls of $M_{S,k}$ and $M_{(S,T),k}$, i.e.~when considering
LSHs.  Let us first formally define an asymmetric $\alpha$-LSH.

Given two collections of objects $S$ and $T$, an alphabet $\Gamma$, a
similarity function $\sm:S\times T \rightarrow [-1,1]$, and
$\alpha>0$, we say that an {\bf $\alpha$-ALSH} is a distribution over
pairs of functions $f:S\rightarrow\Gamma$, $g:T\rightarrow\Gamma$, or
equivalently over $M_{(S,T),\abs{\Gamma}}$, such that for some
$\theta\in\RR$ and all $x\in S,y\in T$:
\begin{equation}
\alpha \EE_{(f,g)\in \FF\times \GG}[\kappa_{f,g}(x,y))] - \theta =\sm(x,y).
\end{equation}

To understand the power of asymmetric LSH, recall that many symmetric
similarity functions do not have an $\alpha$-LSH for any $\alpha$.  On
the other hand, any similarity function over finite domains necessarily
has an $\alpha$-ALSH:
\begin{clm}\label{claim:existence}
  For any similarity function $\sm:S\times T\rightarrow[-1,1]$ over
  finite $S,T$, there exists an $\alpha$-ALSH with $\alpha\leq \min\{|S|,|T|\}$
\end{clm}
This is corollary of Theorem \ref{thm:ineq} that will be proved later in section \ref{s:relaxation}.  The proof follows from  Theorem \ref{thm:ineq}  the following upper bound on the max-norm:
$$
\|Z\|_{\max} \leq \text{rank}(Z).\|Z\|_\infty^2
$$
where $\|Z\|_\infty^2 = \max_{x,y}|Z(x,y)|$.

In section \ref{s:lsh}, we saw that similarity functions that do not
admit an $\alpha$-LSH, still admit Generalized $\alpha$-LSH.  However,
the gap between the $\alpha$ required for a Generalized $\alpha$-LSH
and that required for an $\alpha$-ALSH might be as large as $\Omega(|S|)$:

\begin{theorem}\label{thm:example}
  For any even $n$, there exists a set $S$ of $n$ objects and a
  similarity $Z:S\times S \rightarrow \RR$ such that
\begin{itemize}
\item there is a binary $3K_R$-ALSH for $Z$, where
  $K_R\approx 1.79$ is Krivine's constant;
\item there is no Generalized $\alpha$-LSH for any $\alpha <n-1$.
\end{itemize}
\end{theorem}
\begin{proof}
Let $S=[n]$ and $Z$ be the following similarity matrix:
$$
Z= 2I_{n\times n}+\begin{bmatrix} -1_{\frac{n}{2}\times\frac{n}{2}} & 1_{\frac{n}{2}\times\frac{n}{2}}\\ 1_{\frac{n}{2}\times\frac{n}{2}} & -1_{\frac{n}{2}\times\frac{n}{2}} \end{bmatrix}
$$
Now we use Theorem~\ref{thm:ineq}, which we will prove later (our proof of Theorem~\ref{thm:ineq} does not rely on the proof of this theorem).
Using triangle inequality property of the norm, we have $\|Z\|_{\max}\leq \|Z-2I_{n \times n}\|_{\max} + \|2I_{n \times n}\|_{\max}=3$; and by Theorem~\ref{thm:ineq} there is a $3K_R$-ALSH for $Z$. Looking at the decomposition of $Z$, it is not difficult to see that the smallest eigenvalue of $Z$ is $2-n$. So in order to have a positive semidefinite similarity matrix, we need $\gamma$ to be at least $n-2$ and $\theta$ to be at least $-1$ (otherwise the sum of elements of $Z + \theta + (n-2)I$ will be less than zero and so $Z + \theta + (n-2)I$ will not be positive semidefinite). So $\alpha=\theta+\gamma$ is at least $n-1$.
\end{proof}

\section{Convex Relaxations, $\alpha$-LSH and Max-norm}
\label{s:relaxation}

We now turn to two questions which are really the same: can we get a
tight convex relaxation of the set $M_{(S,T),k}$ of (asymmetric)
clustering incidence functions, and can we characterize the values of
$\alpha$ for which we can get an $\alpha$-ALSH for a particular
similarity measure.

For notational simplicity, we will now fix $S$ and $T$ and use $M_k$
to denote $M_{(S,T),k}$.

\subsection{The Ratio Function}

The tightest possible convex relaxation of $M_k$ is simply its convex
hull $\conv M_k$. Assuming $\mathsf{P}\neq \mathsf{NP}$, $\conv M_k$ is not polynomially tractable. What we ask here is whether he have a tractable tight
relaxation of $\conv M_k$.  To measure tightness of some convex $B
\supseteq M_k$, for each $Z \in B$, we will bound its {\em cluster
  ratio}:
$$
\rho_k(Z) = \min \{r | Z \in r \conv M_k\} = \min \{ r | Z/r \in \conv
M_k \}.
$$
That is, by how much to we have to inflate $M_k$ so that includes
$Z\in B$.  The supremum $\rho_k(B)=\sup_{Z\in B}\rho_k(Z)$ is then the
maximal inflation ratio between $\conv M_k$ and $B$, i.e.~such that
$\conv M_k \subseteq B \subseteq \rho_k \conv M_k$.  Similarly, we
define the {\em centralized cluster ratio} as:
$$
\hat{\rho}_k(Z) = \min_{\theta\in\RR} \min \{r | Z-\theta \in r \conv M_k\}.
$$
This is nothing but the lowest $\alpha$ for which we have an $\alpha$-ALSH:
\begin{clm}
\label{claim:alsh}
For any similarity function $\sm(x,y)$, $\hat{\rho}_k(sim)$ is equal
to the smallest $\alpha$ s.t.~there exists an $\alpha$-ALSH for
$\sm$ over alphabet of cardinality $k$.
\end{clm}
\begin{proof}
We write the problem of minimizing $\alpha$ in $\alpha$-ALSH as:
\begin{equation}
\label{eq:alsh2}
\begin{aligned}
&\min_{\theta\in \RR,\alpha \in \RR^+} 
& &\alpha \\
& \qquad \qquad\;\; \;\;\text{s.t}
& & \sm(x,y)= \alpha \EE_{(f,g)\in \FF \times \GG}[\kappa_{f,g}(x,y)] - \theta
\end{aligned}
\end{equation}
We know that:
$$
\EE_{(f,g)\in \FF \times \GG}[\kappa_{f,g}(x,y)] =  \sum_{f\in M_{S,k}} \sum_{g\in M_{T,k}}\kappa_{f,g}(x,y)p(f,g)
$$
where $p(f,g)$ is the joint probability of hash functions $f$ and $g$.
Define $\mu(f,g) = \alpha p(f,g)$ and write:
$$
\alpha = \alpha\sum_{f\in M_{S,k}}\sum_{g\in M_{T,k}}p(f,g)
=\sum_{f\in M_{S,k}}\sum_{g\in M_{T,k}}\alpha p(f,g)
=\sum_{f\in M_{S,k}}\sum_{g\in M_{T,k}}\mu(f,g)
$$
We have:
\begin{eqnarray*}
\alpha\sum_{f\in M_{S,k}}\sum_{g\in M_{T,k}}\kappa_{f,g}(x,y) p(f,g)-\theta &=&\sum_{f\in M_{S,k}}\sum_{g\in M_{T,k}}\kappa_{f,g}(x,y) \mu(f,g)-\theta\\
\end{eqnarray*}
Substituting the last two equalities into formulation $\ref{eq:alsh2}$ gives us the formulation for centralized cluster ratio.
\end{proof}

Our main goal in this section is to obtain tight bounds on $\rho_k(Z)$
and $\hat{\rho}_k(Z)$.

\paragraph{The Ratio Function and Cluster Norm}

The convex hull $\conv M_k$ is related to the cut-norm, and its
generalization the cluster-norm, and although the two are not
identical, its worth understanding the relationship.

For $k=2$, the ratio function is a norm, and is in fact the dual of a
modified cut-norm:
\begin{equation}
\rho^*_2(W) =\|W\|_{C,2}= \max_{u:S\rightarrow \{\pm 1\},v:S\rightarrow \{\pm 1\}} \sum_{x\in S,y\in T}W(x,y)u(x)v(y)
\end{equation}
The norm $\norm{W}_{C,2}$ is a variant of the cut-norm, and is always
within a factor of four from the cut-norm as defined in,
e.g.~\cite{alon06}. The set $\conv M_2$ in this case is the unit ball of the
modified cut-norm.

For $k>2$, the ratio function is {\em not} a norm, since $M_k$, for
$k>2$, is not symmetric about the origin: we might have $Z\in M_k$ but
$-Z \not\in M_k$ and so $\rho_k(Z)\neq\rho_k(-Z)$.  A ratio function
defined with respect the symmetric convex hull of
$\conv(M_k\cup-M_k)$, is a norm, and is dual to the following {\em
  cluster norm}, which is a generalization of the modified cut-norm:
\begin{equation}
\|W\|_{C,k} = \max_{u:S\rightarrow \Gamma,v:S\rightarrow \Gamma}
\sum_{x\in S,y\in T}W(x,y)\kappa_{u,v}(x,y)
\end{equation}

\subsection{A Tight Convex Relaxation using the Max-Norm}

Recall that the max-norm (also known as the
$\gamma_2:\ell_1\rightarrow\ell_\infty$ norm) of a matrix is defined
as \cite{srebro05}:
$$
\|Z\|_{\max}=\min_{UV^{\top}}\max(\|U\|^2_{2,\infty},\|V\|^2_{2,\infty})
$$
where $\|U\|_{2,\infty}$ is the maximum $\ell_2$ norm of rows of the
matrix $U$. The max-norm is SDP representable and thus tractable
\cite{srebro05b}.  Even when $S$ and $T$ are not finite, and thus $\sm$ is
not a finite matrix, the max-norm can be defined as above, where now
$U$ and $V$ can be thought of as mappings from $S$ and $T$
respectively into a Hilbert space, with
$\sm(x,y)=(UV^{\top})(x,y)=\inner{U(x),V(y)}$ and
$\norm{U}_{2,\infty}=\sup_x \norm{U(x)}$.

We also define the {\em centralized max-norm}, which, even though it
is {\em not} a norm, we denote as:
$$
\|Z\|_{\widehat{\max}} = \min_\theta \|Z-\theta\|_{\max}
$$
The centralized max-norm is also SDP-representable.

Our main result is that the max-norm provides a tight bound on the
ratio function:

\begin{theorem}
\label{thm:ineq}
For any similarity function $\sm:S\times T \rightarrow \RR$ we have that:
$$
\frac{1}{2}\|sim\|_{\widehat{\max}} \leq \frac{1}{2}\hat{\rho}_2(sim)\leq \hat{\rho}(sim) \leq \hat{\rho}_k(sim) \leq \hat{\rho}_2(sim) \leq K \|sim\|_{\widehat{\max}}
$$
and also
$$
\frac{1}{3}\|sim\|_{\max} \leq \rho(sim) \leq \rho_k(sim) \leq \rho_2(sim) \leq K\|sim\|_{\max}
$$
where all inequalities are tight and we have $1.67 \leq K_G \leq K
\leq K_R\leq 1.79$ ($K_G$ is Grothendieck's constant and $K_R$ is Krivine's constant).
\end{theorem}

Considering the dual view of $\rho(\sm)$, the theorem can also be
viewed in two ways: First, we see that the centralized max-norm
provides a tight characterization (up to a small constant factor) of
the smallest $\alpha$ for which we can obtain an $\alpha$-ALSH.  In
particular, since for domains (i.e.~finite matrices) the max-norm is
always finite, this establishes that we always have an $\alpha$-ALSH,
as claimed in Claim \ref{claim:existence}.  We also used it in Theorem \ref{thm:example} to
establish the existence of an $\alpha$-ALSH for a specific, small,
$\alpha$.  

Second, bounding the ratio function establishes that the max-norm ball
is a tight tractable relaxation of $\conv M_k$:
\begin{equation}
  \label{eq:relaxation}
\left\{ Z \;\middle\|\; \norm{Z}_{\max} \leq 1/K \right\} 
\subseteq \conv M_k \subseteq  
\left\{ Z \;\middle\|\; \norm{Z}_{\max} \leq 3 \right\} 
\end{equation}

Third, we see the effect of the alphabet size $k$ (number of clusters)
on the convex hull is very limited.

\paragraph{The Symmetric Case}

It is not difficult to show that the lower bounds for $\alpha$-LSH are
the same as for $\alpha$-ALSH and the inequalities are tight. 
However, there are no upper bounds for $\alpha$-LSH similar to those for $\alpha$-ALSH. 
Specifically, let $\hat \alpha$ and $\hat \alpha_g$ be the smallest values of $\alpha$ such that there is an $\alpha$-LSH for $\sm$ and there is a generalized $\alpha$-LSH for $\sm$, respectively.
Note that for some similarity functions $\sm$ there is no
$\alpha$-LSH at all; that is, $\hat\alpha = \infty$ and $\|\sm\|_{\max}  < \infty$.
Also, as Theorem~\ref{thm:example} shows, there is a similarity function $\sm$ 
such that 
$$\|\sm\|_{\max} = O(1) \quad\text{ but } \quad \hat\alpha_g\geq n-1.$$ 
Moreover, it follows from the result of \cite{arora05} that there is no efficiently computable upper bound $\beta$ for $\hat \alpha_g$ such that
$$\frac{\beta}{\log^c n} \leq \hat \alpha_g \leq \beta$$
(under a standard complexity assumption that $NP\not\subseteq DTIME(n^{\log^3 n})$).
That is, neither the max-norm nor any other efficiently computable norm of
$\sm$ gives a constant factor approximation for $\hat \alpha_g$.

In the remainder of this section we prove a series of lemmas
corresponding to the inequalities in Theorem \ref{thm:ineq}.

\subsection{Proofs}

\begin{lemma}
For any two sets $S$ and $T$ of objects and any function $\sm:S\times T \rightarrow R$, we have that $\hat{\rho}_2(sim) \leq 2\hat{\rho}(sim)$ and the inequality is tight. 
\end{lemma}
\begin{proof}
Using Claim \ref{claim:alsh}, all we need to do is to prove that given the function $\sm$, if there exist an $\alpha$-ALSH with arbitrary cardinality, then we can find a binary $2\alpha-ALSH$. In order to do so, we assume that there exists an $\alpha$-ALSH for family $\FF$ and $\GG$ of hash functions such that:
$$
\alpha \EE_{(f,g)\in \FF \times \GG}[ \kappa_{f,g}(x,y)]= sim(x,y)+\theta
$$
where $f:S\rightarrow \Gamma$ and $g:T\rightarrow \Gamma$ are hash functions. Now let $\HH$ be a family of pairwise independent hash functions of the form $\Gamma \rightarrow \{\pm 1\}$ such that each element $\gamma \in \Gamma$, has the equal chance of being mapped into -1 or 1. Now, we have that:
\begin{eqnarray*}
2\alpha \EE_{h\in \HH,(f,g)\in \FF \times \GG}[ \kappa_{hof,hog}(x,y)] &=& 2\alpha \EE_{h\in \HH,(f,g)\in \FF \times \GG}[ \kappa_{hof,hog}(x,y)]\\
&=& 2\alpha \EE_{h\in \HH,(f,g)\in \FF \times \GG}[ h(f(x))h(g(y))]\\
&=& 2\alpha (2P_{h\in \HH,(f,g)\in \FF \times \GG}[ h(f(x))=h(g(y))]-1)\\
&=& 2\alpha P_{(f,g)\in \FF \times \GG}[ f(x)=g(y)]\\
&=& sim(x,y)+\theta +\alpha\\
&=& sim(x,y)+\tilde{\theta}\\
\end{eqnarray*}
The tightness can be demonstrated by the example $sim(x,y) = 2_{x=y}-1$ when $S$ is not finite.
\end{proof}

\begin{lemma}
For any two sets $S$ and $T$ of objects and any function $\sm:S\times T \rightarrow R$, we have that $\|sim\|_{\max} \leq \rho_2(sim)$ and the inequality is tight.
\end{lemma}
\begin{proof}
Without loss of generality, we assume that $\Gamma=\{\pm 1\}$. We want to solve the following optimization problem:
\begin{eqnarray}
\nonumber
 \rho_2(sim) = \min_{\mu : M_{S,2} \times M_{T,2} \rightarrow \RR^+} && \sum_{f\in M_{S,2}}\sum_{g\in M_{T,2}}\mu(f,g)\\ 
\nonumber
\text{s.t.}&& \sm(x,y) =  \sum_{f\in M_{S,2}}\sum_{g\in M_{T,2}}\kappa_{f,g}(x,y)\mu(f,g)\\\
\nonumber
\end{eqnarray}
For any $x\in S$ and $y\in T$, we define two new function variables $\ell_x:M_{S,2} \times M_{T,2} \rightarrow \RR$ and $r_y:M_{S,2} \times M_{T,2} \rightarrow \RR$:
\begin{eqnarray*}
\nonumber
\ell_x(f,g)&=&\sqrt{\mu(f,g)}f(x)\\
r_y(f,g)&=&\sqrt{\mu(f,g)}g(y)
\end{eqnarray*}
Since cluster incidence matrix can be written as $\kappa_{f,g}(x,y) = f(x)g(y)$, we have $\sm(x,y) = \inner{\ell_x,r_y}$ and $\|\ell_x\|_2^2 = \sum_{f\in M_{S,2}}\sum_{g\in M_{T,2}}\mu(f,g)$.
Therefore, we rewrite the optimization problem as:
\begin{eqnarray}
\nonumber
 \rho_2(sim) =\min_{t,\ell,r,\mu : M_{S,2} \times M_{T,2} \rightarrow \RR^+} && t\\ 
\nonumber
\text{s.t.}&& \inner{l_x,r_y}=\sm(x,y)\\
\nonumber
&&\|\ell_x\|_2^2 \leq t\\
\nonumber
&&\|r_y\|_2^2 \leq t\\
\nonumber
&&\ell_x(f,g)=\sqrt{\mu(f,g)}f(x)\\
\nonumber
&&r_y(f,g)=\sqrt{\mu(f,g)}g(y)\\
\nonumber
\end{eqnarray}
Finally, we relax the above problem by removing the last two constraints:
\begin{eqnarray}
\label{eq:mr}
\nonumber
\|sim\|_{\max} = \min_{t,\ell,r} &&t\\ 
\nonumber
\text{s.t.}&&\inner{l_x,r_y}=\sm(x,y)\\
&&\|\ell_x\|_2^2 \leq t\\
\nonumber
&&\|r_x\|_2^2 \leq t \\
\nonumber
\end{eqnarray}
The above problem is a max-norm problem and the solution is $\|sim\|_{\text{max}}$. Therefore, $\|sim\|_{\text{max}} \leq \rho_2(sim)$. Taking the function $sim(x,y)$ to be a binary cluster incidence function will indicate the tightness of the inequality.
\end{proof}

\begin{lemma}
\label{lem:naor}
(Krivine's lemma \cite{krivine77}) For any two sets of unit vectors $\{u_i\}$ and $\{v_j\}$ in a Hilbert space $H$, there are two sets of unit vectors $\{u'_i\}$ and $\{v'_j\}$ in a Hilbert space $H'$ such that for any $u_i$ and $v_j$, $\sin(c \inner{u_i,v_j}) = \inner{u'_i,v'_j}$ where $c=\sinh^{-1}(1)$.
\end{lemma}
\begin{lemma}
\label{lem:kr}
For any two sets $S$ and $T$ of objects and any function $\sm:S\times T \rightarrow R$, we have that $\rho_2(sim) \leq K\|sim\|_{\max}$ where $1.67 \leq K_G \leq K \leq K_R\leq 1.79$ ($K_G$ is Grothendieck's constant and $K_R$ is Krivine's constant).
\end{lemma}
\begin{proof}
A part of the proof is similar to \cite{alon06}. Let $\ell_x$ and $r_y$ be the solution to the max-norm formulation \ref{eq:mr}. If we use Lemma \ref{lem:naor} on the normalized $\ell_x/\|\ell_x\|_2$ and $r_y/\|r_y\|_2$ in Hilbert space $H$ and we call the new vectors $\ell'_x$ and $r'_y$ in Hilbert space $H'$, we have that:
$$
\sin\bigg(\frac{c .Z(x,y)}{\|\ell_x\|_2\|r_x\|_2}\bigg) = \inner{\ell'_x,r'_y}
$$
If $z$ is a random vector chosen uniformly from $H'$, by Lemma \ref{lem:naor}, we have:
$$
\EE([\text{sign}(\inner{\ell'_x,z})].[\text{sign}(\inner{r_y',z})]) = \frac{2}{\pi}\arcsin(\inner{\ell'_x,r'_y})) =  \frac{2c}{\pi \|\ell_x\|_2\|r_y\|_2}\sm(x,y)
$$
Now if we set the hashing function $f(x)= s(x).[\text{sign}(\inner{\ell'_x,z})]$ where $s(x)=1$ with probability $\frac{1}{2}+\frac{\|\ell_x\|_2}{2\sqrt{t}}$ and $s(x)=-1$ with probability $\frac{1}{2}-\frac{\|\ell_x\|_2}{2\sqrt{t}}$ we have that:
\begin{eqnarray}
\nonumber
\EE[f(x).\text{sign}(\inner{r_y',z})] &=& \bigg(\frac{1}{2}+\frac{\|\ell_x\|_2}{2\sqrt{t}}\bigg)\frac{2c}{\pi \|\ell_x\|_2\|r_y\|_2}\sm(x,y)\\
\nonumber
&-&\bigg(\frac{1}{2}-\frac{\|\ell_x\|_2}{2\sqrt{t}}\bigg)\frac{2c}{\pi \|\ell_x\|_2\|r_y\|_2}\sm(x,y)\\
\nonumber
&=& \frac{2c}{\pi \sqrt{t}\|r_y\|_2}\sm(x,y)
\nonumber
\end{eqnarray}
If we do the same procedure on $g(y)=s'(x).[\text{sign}(\inner{r'_y,z})]$, we will have:
$$
\EE[f(x).g(y)] = \frac{2c}{t\pi}\sm(x,y)
$$
By setting $\mu(f,g)=\frac{\pi\|sim\|_{\max}}{2c}p(f,g)$ where $p(f,g)$ is the probability distribution over the defined $f$ and $g$, we can see that such $\mu(f,g)$ is a feasible solution for the formulation of cluster ratio and we have:
$$
\rho_2(sim)\leq  \sum_{f\in M_{S,2}}\sum_{g\in M_{T,2}}\mu(f,g) = \frac{\pi}{2c}\|sim\|_{\max} = K_R \|sim\|_{\max}
$$
The inequality $K_G\leq K$ is known due to \cite{alon06}. 
\end{proof}

\bibliographystyle{plain}
\bibliography{ref}

\appendix


\section{Random Matrices}
\label{s:random}
In this section we investigate the locality sensitive hashing schemes on random p.s.d matrices. We generate a random $n \times n$ positive semidefinite matrix $Z$ of rank at most $d$ by choosing $n$ $d$-dimensional unit vectors $x^{(i)}$ uniformly at random from the unit ball and set $Z_{ij}=\inner{x^{(i)},x^{(j)}}$. Since we are generating the data randomly and $\EE[Z_{ij}]=0$, we don't expect to observe major changes by thresholding the matrix. So our analysis is limited to the LSH without thresholding, i.e. $\theta = 0$.

Since based on Theorem \ref{thm:ineq}, we already know given any set of unit vectors $x^{(1)},\dots,x^{(n)}$ and $Z_{ij}=\inner{x^{(i)},x^{(j)}}$, there is an $K_R$-ALSH for the matrix $Z$, we are just interested in investigating the symmetric LSH for these random vectors.
\subsection{LSH}
For the symmetric LSH, we only have two possibilities: either having $LSH$ with $\alpha=1$ or not having any LSH. We also know from Claim \ref{claim:nolsh} that there is no $\alpha$ LSH if $d< \log_2 n$ because in that case $D(x^{(i)},x^{(j)}) = 1- Z_{ij}$ is not metric. So we want to know the conditions under which the distance will be a metric and also the conditions for having $\alpha$-LSH with high probability.

\begin{lemma}
\label{lem:dens}
\cite{dasgupta03}
If $x$ is a $d$-dimensional unit vector and $\tilde{x}$ is its projection onto another unit vector that is sampled uniformly at random from the unit sphere, then for any $t>1$, we have $\EE[\|\tilde{x}\|_2^2]=\frac{1}{d}$ and moreover, $\PP( \|\tilde{x}\|_2^2 \geq \frac{t}{d}) \leq e^{\frac{1-t + \log t}{2}}$.
\end{lemma}
\begin{lemma}
Let $\{x^{(1)},\dots,x^{(n)}\}$ be a set of unit vectors sampled uniformly at random from the unit sphere and for any $1\leq i,j \leq n$ let $Z_{ij}=\inner{x^{(i)},x^{(j)}}$. If $d \geq 72\log_e n+\log_e\frac{1}{\delta}$, then the distance measure $\Delta_{ij}=1-Z_{ij}$ is metric with probability at least $1-\delta$.
\end{lemma}
\begin{proof}
The distance measure $\Delta_{ij} = 1-\inner{x^{(i)},x^{(j)}}$ is not a metric if and only if there exist $i$, $j$ and $k$ such that 
$$
(1-\inner{x^{(i)},x^{(j)}}) + (1-\inner{x^{(i)},x^{(k)}}) < (1-\inner{x^{(j)},x^{(k)}})
$$
A simple reordering of the above inequality gives us:
$$
C_{ijk}=\inner{x^{(i)},x^{(j)}} + \inner{x^{(i)},x^{(k)}} - \inner{x^{(j)},x^{(k)}}) > 1
$$
For this inequality to hold, the absolute value of at least one of the inner products $\inner{x^{(i)},x^{(j)}}$, $\inner{x^{(i)},x^{(k)}}$, $\inner{x^{(j)},x^{(k)}}$ must be at least $\frac{1}{3}$. Now we have:
\begin{eqnarray}
\nonumber
\PP(\Delta \text{ is not a metric}) &=& \PP(\exists_{ijk} : C_{ijk}> 1)\\
\nonumber
&\leq& \PP(\exists_{ij} : |\inner{x^{(i)},x^{(j)}}|>1/3)\\
\nonumber
&\leq& \frac{n^2}{2} \, \PP(|\inner{x^{(1)},x^{(2)}}|>1/3)\\
\nonumber
\end{eqnarray}
Since both $x^{(1)}$ and $x^{(2)}$ are random vectors, the probability $\PP(|\inner{x^{(1)},x^{(2)}}|>1/3)$ is equal to the probability that the projection of a random $d$-dimensional vector onto a 1-dimensional subspace is at least $1/3$ in absolute value. By Lemma~\ref{lem:dens}, we have:
\begin{eqnarray}
\nonumber
\PP(\Delta \text{ is not a metric}) &\leq& \frac{n^2}{2}\, \PP(|\inner{x^{(1)},x^{(2)}}|>1/3)\\
\nonumber
&\leq& \frac{n^2}{2}\, \PP(\inner{x^{(1)},x^{(2)}}^2>1/9)\\
\nonumber
&\leq& \frac{n^2}{2}e^{\frac{1+\log(d/9)-(d/9)}{2}}\\
\nonumber
&\leq& n^2e^{-\frac{d}{36}}\\
\nonumber
&\leq& \delta\\
\nonumber
\end{eqnarray}
\end{proof}
\begin{lemma}(\cite{ver10}, Theorem 5.39)
\label{lem:eig}
Let $\{x^{(1)},\dots,x^{(n)}\}$ be a set of unit vectors sampled uniformly at random from the unit sphere and $t \in (0,1)$.
Let $Z_{ij}=\inner{x^{(i)},x^{(j)}}$ for $1\leq i,j \leq n$. 
If $d \geq C_1 n/t^2 $ then with probability at least $1-2 e^{-C_2 t^2 N}$,  we have $|\lambda_{i} -1| \leq t$ for all eigenvalues $\lambda_i$ of $Z$.
Here, $C_1 > 0$ and $C_2 > 0$ are some absolute constants.
\end{lemma}
\begin{theorem}\label{thm:randomLSH}
Let $\{x^{(1)},\dots,x^{(n)}\}$ be a set of unit vectors sampled uniformly at random from the unit sphere. Let $Z_{ij}=\inner{x^{(i)},x^{(j)}}$
for $1\leq i,j \leq n$. If $d\geq C n \log^2 n$ then with probability at least $1-e^{C' n/\log^2 n}$, there is an LSH for $Z$. Here, $C>0$ and $C'>0$ are some absolute constants.
\end{theorem}
\begin{proof}
Apply Lemma~\ref{lem:eig} with $t=\frac{1}{C_0\log n}$ (where $C_0$ is a sufficiently large constant). We get that if $d \geq (C_0^2 C_1)\, n \log^2 n$ then with probability at least $1- e^{-(C_2/C_1^2)N/\log^2 n}$ the smallest eigenvalue is greater than or equal to $1-\frac{1}{C\log n}$. Therefore, matrix $Y = C\log n \, (Z - (1-\frac{1}{C\log n})I)$ is a positive semidefinite matrix with unit diagonal. Now according to~\cite{charikar04}, there exists a distribution over a family $\HH$ of hash functions such that for any $i\neq j$,
$
E_{h \in \HH}[h_ih_j] = \frac{Y_{ij}}{C\log n}
$.
We have,
$$
E_{h \in \HH}[h_ih_j] = \frac{Y_{ij}}{C\log n}
= Z_{ij} -\left(1-\frac{1}{C\log n}\right)I_{ij}
= Z_{ij}
$$
Moreover, for every $i$, we have $E_{h \in \HH}[h_ih_i]=1=Z_{ii}$.
\end{proof}
\subsection{Generalized LSH}
In this section, we try to investigate the conditions to have Generalized LSH with high probability. 
\begin{lemma}
Let $\{x^{(1)},\dots,x^{(n)}\}$ be a set of unit vectors.
Let $Z_{ij}=\inner{x^{(i)},x^{(j)}}$ for $1\leq i,j \leq n$. 
There is a generalized $\alpha$-LSH for matrix $Z$ with $\alpha = O(\log n)$.
\end{lemma}
\begin{proof}
Matrix $Z$ is positive semi-definite and thus its smallest eigenvalue $\lambda_{min}$ is non-negative. 
Applying Claim~\ref{claim:ALSHexist}, we get the statement of the lemma.
\end{proof}
\begin{theorem}
Let $\{x^{(1)},\dots,x^{(n)}\}$ be a set of unit vectors sampled uniformly at random from the unit sphere, let $0 <\alpha < O(\log n)$. Let $Z_{ij}=\inner{x^{(i)},x^{(j)}}$
for $1\leq i,j \leq n$. 
If $d\geq C n \log^2 n / \alpha^2$ then with probability at least $1-e^{C' n \alpha^2/\log^2 n}$, there is a generalized $\alpha$-LSH for $Z$. Here $C>0$ and $C'>0$ are some absolute constants.
\end{theorem}
\begin{proof}
The proof is a straightforward generalization of Theorem~\ref{thm:randomLSH}.
\end{proof}

\end{document}